\documentclass[letterpaper,draft]{article}
\usepackage{ijcai13}
\usepackage{times}
\usepackage{url}
\usepackage{amsmath}
\usepackage{amssymb}
\usepackage{amsthm}
\usepackage{multirow}
\usepackage{graphicx}
\usepackage{algorithm} 
\usepackage[noend]{algpseudocode}
\usepackage{xspace}
\usepackage{xcolor}
\usepackage{thmtools}
\usepackage{thm-restate}
\usepackage{ifdraft}
\usepackage{macros}
\allowdisplaybreaks 

\floatname{algorithm}{Procedure}

\declaretheorem[name=Theorem]{theorem}
\declaretheorem[name=Definition,sibling=theorem]{definition}
\declaretheorem[name=Proposition,sibling=theorem]{proposition}
\declaretheorem[name=Example,sibling=theorem]{example}

\title{Computing Datalog Rewritings Beyond Horn Ontologies\thanks{This work
was supported by the Royal Society, the EPSRC projects Score!, Exoda, and
MaSI$^3$, and the FP7 project OPTIQUE.}}

\author{Bernardo Cuenca Grau$^1$ \and Boris Motik$^1$ \and Giorgos Stoilos$^2$ \and Ian Horrocks$^1$ \\
{}$^1$\,University of Oxford, UK \hspace{2cm} {}$^2$\,NTU Athens, Greece}

\begin{document}

\maketitle

\begin{abstract}
Rewriting-based approaches for answering queries over an OWL 2 DL ontology
have so far been developed mainly for Horn fragments of OWL 2 DL. In this
paper, we study the possibilities of answering queries over non-Horn
ontologies using datalog rewritings. We prove that this is impossible in
general even for very simple ontology languages, and even if $\Ptime = \NP$.
Furthermore, we present a resolution-based procedure for $\SHI$ ontologies
that, in case it terminates, produces a datalog rewriting of the ontology. We
also show that our procedure necessarily terminates on $\Bool$ ontologies---an
extension of OWL 2 QL with transitive roles and Boolean connectives.
\end{abstract}

\section{Introduction}

Answering conjunctive queries (CQs) over OWL 2 DL ontologies is a
computationally hard \cite{GHLS08a,DBLP:conf/cade/Lutz08}, but key problem in
many applications. Thus, considerable effort has been devoted to the
development of OWL 2 DL fragments for which query answering is tractable in
\emph{data complexity}, which is measured in the size of the data only. Most
languages obtained in this way are \emph{Horn}: ontologies in such languages
can always be translated into first-order Horn clauses. This includes the
families of `lightweight' languages such as DL-Lite \cite{CDLLR07b},
$\mathcal{EL}$ \cite{DBLP:conf/ijcai/BaaderBL05}, and DLP
\cite{DBLP:conf/www/GrosofHVD03} that underpin the QL, EL, and RL profiles of
OWL 2, respectively, as well as more expressive languages, such as
Horn-$\mathcal{SHIQ}$ \cite{DBLP:conf/ijcai/HustadtMS05} and
Horn-$\mathcal{SROIQ}$ \cite{DBLP:conf/ijcai/OrtizRS11}.

Query answering can sometimes be implemented via query rewriting: a rewriting
of a query $\Q$ w.r.t.\ an ontology $\T$ is another query $\Q'$ that captures
all information from $\T$ necessary to answer $\Q$ over an arbitrary data set.
Unions of conjunctive queries (UCQs) and datalog are common target languages
for query rewriting. They ensure tractability w.r.t.\ data complexity, while
enabling the reuse of optimised data management systems: UCQs can be answered
using relational databases \cite{CDLLR07b}, and datalog queries can be
answered using rule-based systems such as OWLim \cite{bishop2011owlim} and
Oracle's Semantic Data Store \cite{Wu08}. Query rewriting algorithms have so
far been developed mainly for Horn fragments of OWL 2 DL, and they have been
implemented in systems such as QuOnto \cite{DBLP:conf/aaai/AcciarriCGLLPR05},
Rapid \cite{Chortaras11}, Presto \cite{DBLP:conf/kr/RosatiA10}, Quest
\cite{DBLP:conf/kr/Rodriguez-MuroC12}, Clipper
\cite{DBLP:conf/aaai/EiterOSTX12}, Owlgres \cite{DBLP:conf/owled/StockerS08},
and Requiem \cite{Hector10a}.

Horn fragments of OWL 2 DL cannot capture \emph{disjunctive knowledge}, such
as `every student is either an undergraduate or a graduate'. Such knowledge
occurs in practice in ontologies such as the NCI Thesaurus and the
Foundational Model of Anatomy, so these ontologies cannot be processed using
known rewriting techniques; furthermore, no query answering technique we are
aware of is tractable w.r.t.\ data complexity when applied to such ontologies.
These limitations cannot be easily overcome: query answering in even the basic
non-Horn language $\ELU$ is $\coNP$-hard w.r.t.\ data complexity
\cite{DBLP:conf/lpar/KrisnadhiL07}, and since answering datalog queries is
$\Ptime$-complete, it may not be possible to rewrite an arbitrary $\ELU$
ontology into datalog unless ${\Ptime = \NP}$. Furthermore,
\citeA{Lutz:2012ug} showed that tractability w.r.t.\ data complexity cannot be
achieved for an arbitrary non-Horn ontology $\T$ with `real' disjunctions: for
each such $\T$, a query $\Q$ exists such that answering $\Q$ w.r.t.\ $\T$ is
$\coNP$-hard.

The result by \citeA{Lutz:2012ug}, however, depends on an interaction between
existentially quantified variables in $\Q$ and disjunctions in $\T$. Motivated
by this observation, we consider the problem of computing datalog rewritings
of \emph{ground} queries (i.e., queries whose answers must map all the
variables in $Q$ to constants) over non-Horn ontologies. Apart from allowing
us to overcome the negative result by \citeA{Lutz:2012ug}, this also allows us
to compute a rewriting of $\T$ that can be used to answer an arbitrary ground
query. Such queries form the basis of SPARQL, which makes our results
practically relevant. We summarise our results as follows.

In Section \ref{sec:NegativeResults}, we revisit the limits of datalog
rewritability for a language as a whole and show that non-rewritability of
$\ELU$ ontologies is independent from any complexity-theoretic assumptions.
More precisely, we present an $\ELU$ ontology $\T$ for which query answering
cannot be decided by a family of monotone circuits of polynomial size, which
contradicts the results by \citeA{Afrati:1995un}, who proved that fact
entailment in a fixed datalog program can be decided using monotone circuits
of polynomial size. Thus, instead of relying on complexity arguments, we
compare the lengths of proofs in $\ELU$ and datalog and show that the proofs
in $\ELU$ may be considerably longer than the proofs in datalog.

In Section \ref{sec:DatalogRewritings}, we present a three-step procedure that
takes a $\SHI$-ontology $\T$ and attempts to rewrite $\T$ into a datalog
program. First, we use a novel technique to rewrite $\T$ into a TBox
$\Omega_\T$ without transitivity axioms while preserving entailment of
\emph{all} ground atoms; this is in contrast to the standard techniques (see,
e.g., \cite{hms07reasoning}), which preserve entailments only of unary facts
and binary facts with roles not having transitive subroles. Second, we use the
algorithm by \citeA{hms07reasoning} to rewrite $\Omega_\T$ into a
\emph{disjunctive datalog program} $\DD(\Omega_\T)$. Third, we adapt the
knowledge compilation technique by \citeA{DBLP:journals/ai/Val05} and
\citeA{selman1996knowledge} to transform $\DD(\Omega_\T)$ into a datalog
program. The final step is not guaranteed to terminate in general; however, if
it terminates, the resulting program is a rewriting of $\T$.

In Section \ref{sec:Termination}, we show that our procedure always terminates
if $\T$ is a $\Bool$-ontology---a practically-relevant language that extends
OWL 2 QL with transitive roles and Boolean connectives.
\citeA{Artale09thedl-lite} proved that the data complexity of concept queries
in this language is tractable (i.e., $\NLogSpace$-complete). We extend this
result to all ground queries and thus obtain a goal-oriented rewriting
algorithm that may be suitable for practical use.

Our technique, as well as most rewriting techniques known in the literature,
is based on a sound inference system and thus produces only \emph{strong
rewritings}---that is, rewritings entailed by the original ontology. In
Section \ref{sec:LimitsStrong} we show that non-Horn ontologies exist that can
be rewritten into datalog, but that have no strong rewritings. This highlights
the limits of techniques based on sound inferences. It is also surprising
since all known rewriting techniques for Horn fragments of OWL 2 DL known to
us produce only strong rewritings.

The proofs of all of our technical results are given in \ifdraft{appendices
\ref{sec:Proofs-NegativeResults}--\ref{sec:Proofs-LimitsStrong}}{the
accompanying technical report \cite{TR}}.

\section{Preliminaries}\label{sec:Preliminaries}

We consider first-order logic without equality and function symbols.
Variables, terms, (ground) atoms, literals, formulae, sentences,
interpretations ${I = (\Delta^I, \cdot^I)}$, models, and entailment
($\models$) are defined as usual. We call a finite set of facts (i.e., ground
atoms) an \emph{ABox}. We write $\varphi(\vec{x})$ to stress that a
first-order formula $\varphi$ has free variables ${\vec{x} = x_1, \ldots,
x_n}$.

\subsubsection{Resolution Theorem Proving}

We use the standard notions of (Horn) clauses, substitutions (i.e., mappings
of variables to terms), and most general unifiers (MGUs). We often identify a
clause with the set of its literals. \emph{Positive factoring} (PF) and
\emph{binary resolution} (BR) are as follows, where $\sigma$ is the MGU of
atoms $A$ and $B$:
\begin{equation*}
    \text{PF:} \qquad  \frac{C \vee A \vee B}{C\sigma \vee A\sigma} \qquad \text{BR:} \qquad  \frac{C \vee A \quad D \vee \neg B}{(C \vee D)\sigma}
\end{equation*}
A clause $C$ is a \emph{tautology} if it contains literals $A$ and ${\neg A}$.
A clause $C$ subsumes a clause $D$ if a substitution $\sigma$ exists such that
each literal in $C\sigma$ occurs in $D$. Furthermore, $C$ $\theta$-subsumes
$D$ if $C$ subsumes $D$ and $C$ has no more literals than $D$. Finally, $C$ is
\emph{redundant} in a set of clauses $\S$ if $C$ is a tautology or if $C$ is
$\theta$-subsumed by another clause in $\S$.

\subsubsection{Datalog and Disjunctive Datalog}

A \emph{disjunctive rule} $r$ is a function-free first-order sentence of the
form ${\forall \vec{x} \forall \vec{z}.[\varphi(\vec{x},\vec{z}) \rightarrow
\psi(\vec{x})]}$, where tuples of variables $\vec{x}$ and $\vec{z}$ are
disjoint, $\varphi(\vec{x},\vec{z})$ is a conjunction of atoms, and
$\psi(\vec{x})$ is a disjunction of atoms. Formula $\varphi$ is the
\emph{body} of $r$, and formula $\psi$ is the \emph{head} of $r$. For brevity,
we often omit the quantifiers in a rule. A \emph{datalog rule} is a
disjunctive rule where $\psi(\vec{x})$ is a single atom. A (disjunctive)
datalog program $\P$ is a finite set of (disjunctive) datalog rules. Rules
obviously correspond to clauses, so we sometimes abuse our definitions and use
these two notions as synonyms. The \emph{evaluation} of $\P$ over an ABox $\A$
is the set $\P(\A)$ of facts entailed by ${\P \cup \A}$.

\subsubsection{Ontologies and Description Logics}

A DL \emph{signature} is a disjoint union of sets of \emph{atomic concepts},
\emph{atomic roles}, and \emph{individuals}. A \emph{role} is an atomic role
or an \emph{inverse role} $R^-$ for $R$ an atomic role; furthermore, let
${\inv{R} = R^-}$ and ${\inv{R^-} = R}$. A \emph{concept} is an expression of
the form $\top$, $\bot$, $A$, ${\neg C}$, ${C_1 \sqcap C_2}$, ${C_1 \sqcup
C_2}$, ${\exists R.C}$, ${\forall R.C}$, or $\eself{R}$, where $A$ is an
atomic concept, $C_{(i)}$ are concepts, and $R$ is a role. Concepts
$\eself{R}$ correspond to atoms $R(x,x)$ and are typically not included in
$\SHI$; however, we use this minor extension in Section
\ref{sec:Transitivity}. A $\SHI$-TBox $\T$, often called an \emph{ontology},
is a finite set of axioms of the form ${R_1 \sqsubseteq R_2}$ (\emph{role
inclusion axioms} or RIAs), $\Tra{R}$ (\emph{transitivity axioms}), and ${C_1
\sqsubseteq C_2}$ (\emph{general concept inclusions} or GCIs), where $R_{(i)}$
are roles and $C_{(i)}$ are concepts. Axiom ${C_1 \equiv C_2}$ abbreviates
${C_1 \sqsubseteq C_2}$ and ${C_2 \sqsubseteq C_1}$. Relation $\roleh{\T}$ is
the smallest reflexively--transitively closed relation such that ${R
\roleh{\T} S}$ and ${\inv{R} \roleh{\T} \inv{S}}$ for each ${R \sqsubseteq S
\in \T}$. A role $R$ is \emph{transitive} in $\T$ if ${\Tra{R} \in \T}$ or
${\Tra{\inv{R}} \in \T}$. Satisfaction of a $\SHI$-TBox $\T$ in an
interpretation ${I = (\Delta^I, \cdot^I)}$, written ${I \models \T}$, is
defined as usual \cite{DBLP:conf/dlog/2003handbook}.

An $\ALCHI$-TBox is a $\SHI$-TBox with no transitivity axioms. An $\ELU$-TBox
is an $\ALCHI$-TBox with no role inclusion axioms, inverse roles, concepts
$\eself{R}$, or symbols $\bot$, $\forall$, and $\neg$. A $\Bool$-TBox is a
$\SHI$-TBox that does not contain concepts of the form $\forall R.C$, and
where ${C = \top}$ for each concept of the form ${\exists R.C}$. The notion of
\emph{acyclic} TBoxes is defined as usual \cite{DBLP:conf/dlog/2003handbook}.

A $\SHI$-TBox $\T$ is \emph{normalised} if $\forall$ does not occur in $\T$,
and $\exists$ occurs in $\T$ only in axioms of the form ${\exists R.C
\sqsubseteq A}$, ${\eself{R} \sqsubseteq A}$, ${A \sqsubseteq \exists R.C}$,
or ${A \sqsubseteq \eself{R}}$. Each $\SHI$-TBox $\T$ can be transformed in
polynomial time into a normalised $\SHI$-TBox that is a model-conservative
extension of $\T$.

\subsubsection{Queries and Datalog Rewritings}

A \emph{ground query} (or just a \emph{query}) $\Q(\vec{x})$ is a conjunction
of function-free atoms. A substitution $\sigma$ mapping $\vec{x}$ to constants
is an \emph{answer} to $\Q(\vec{x})$ w.r.t.\ a set $\F$ of first-order
sentences and an ABox $\A$ if ${\F \cup \A \models \Q(\vec{x})\sigma}$;
furthermore, $\cert{\Q}{\F}{\A}$ is the set of all answers to $\Q(\vec{x})$
w.r.t.\ $\F$ and $\A$.

Let $\Q$ be a query. A datalog program $\P$ is a \emph{$\Q$-rewriting} of a
finite set of sentences $\F$ if ${\cert{\Q}{\F}{\A} = \cert{\Q}{\P}{\A}}$ for
each ABox $\A$. The program $\P$ is a \emph{rewriting of $\F$} if $\P$ is a
$\Q$-rewriting of $\F$ for each query $\Q$. Such rewritings are \emph{strong}
if, in addition, we also have ${\F \models \P}$.

\section{The Limits of Datalog Rewritability}\label{sec:NegativeResults}

Datalog programs can be evaluated over an ABox $\A$ in polynomial time in the
size of $\A$; hence, a $\coNP$-hard property of $\A$ cannot be decided by
evaluating a fixed datalog program over $\A$ unless ${\Ptime = \NP}$.
\citeA{DBLP:conf/lpar/KrisnadhiL07} showed that answering ground queries is
$\coNP$-hard in data complexity even for acyclic TBoxes expressed in
$\ELU$---the simplest non-Horn extension of the basic description logic
$\mathcal{EL}$. Thus, under standard complexity-theoretic assumptions, an
acyclic $\ELU$-TBox and a ground query $\Q$ exist for which there is no
$\Q$-rewriting of $\T$. In this section, we show that this holds even if
${\Ptime = \NP}$.

\begin{restatable}{theorem}{nonRewritableGeneralELU}\label{th:non-Rewritable-General-ELU}
    An acyclic $\ELU$-TBox $\T$ and a ground CQ $\Q$ exist such that $\T$ is
    not $\Q$-rewritable.
\end{restatable}

Our proof uses several notions from circuit complexity \cite{Wegener87}, and
results of this flavour compare the sizes of proofs in different formalisms;
thus, our result essentially says that proofs in $\ELU$ can be significantly
longer than proofs in datalog. Let $<$ be the ordering on Boolean values
defined by ${\false < \true}$; then, a Boolean function $f$ with $n$ inputs is
\emph{monotone} if ${f(x_1, \ldots, x_n) \leq f(y_1, \ldots, y_n)}$ holds for
all $n$-tuples of Boolean values ${x_1, \ldots, x_n}$ and ${y_1, \ldots, y_n}$
such that ${x_i \leq y_i}$ for each ${1 \leq i \leq n}$. A decision problem
can be seen as a family of Boolean functions ${\{ f_n \}}$, where $f_n$
decides membership of each $n$-bit input. If each function $f_n$ is monotone,
then $f_n$ can be realised by a monotone Boolean circuit $C_n$ (i.e., a
circuit with $n$ input gates where all internal gates are AND- or OR-gates
with unrestricted fan-in); the size of $C_n$ is the number of its edges. The
family of circuits ${\{ C_n \}}$ corresponding to ${\{ f_n \}}$ has polynomial
size if a polynomial $p(x)$ exists such that the size of each $C_n$ is bounded
by $p(n)$.

We recall how non-3-colorability of an undirected graph $G$ with $s$ vertices
corresponds to monotone Boolean functions. The maximum number of edges in $G$
is ${m(s) = s(s-1)/2}$, so graph $G$ is encoded as a string $\vec{x}$ of
$m(s)$ bits, where bit $x_{i,j}$, ${1 \leq i < j \leq s}$, is $\true$ if and
only if $G$ contains an edge between vertices $i$ and $j$. The
non-3-colorability problem can then be seen as a family of Boolean functions
${\{ f_{m(s)} \}}$, where function $f_{m(s)}$ handles all graphs with $s$
vertices and it evaluates to $\true$ on an input $\vec{x}$ iff the graph
corresponding to $\vec{x}$ is non-3-colourable. Functions $f_n$ such that ${n
\neq m(s)}$ for all $s$ are irrelevant since no graph is encoded using that
many bits.

We prove our claim using a result by \citeA{Afrati:1995un}: if a decision
problem cannot be solved using a family of monotone circuits of polynomial
size, then the problem also cannot be solved by evaluating a fixed datalog
program, regardless of the problem's complexity. We restate the result as
follows.

\begin{restatable}{theorem}{afrati}[Adapted from \citeauthor{Afrati:1995un} \citeyear{Afrati:1995un}]\label{th:datalog-inexpressibility}
    \quad \phantom{a}

    \begin{enumerate}
        \item Let $\P$ be a fixed datalog program, and let $\alpha$ be a fixed
        fact. Then, for an ABox $\A$, deciding ${\P \cup \A \models \alpha}$
        can be solved by monotone circuits of polynomial size.

        \item The non-3-colorability problem cannot be solved by monotone
        circuits of polynomial size.
    \end{enumerate}
\end{restatable}

To prove Theorem \ref{th:non-Rewritable-General-ELU}, we present a TBox $\T$
and a ground CQ $\Q$ that decide non-3-colorability of a graph encoded as an
ABox. Next, we present a family of monotone Boolean functions ${\{ g_{n(u)}
\}}$ that decide answering $\Q$ w.r.t.\ $\T$ an arbitrary ABox $\A$. Next, we
show that a monotone circuit for arbitrary $f_{m(s)}$ can be obtained by a
size-preserving transformation from a circuit for some $g_{n(u)}$; thus, by
Item 2 of Theorem \ref{th:datalog-inexpressibility}, answering $\Q$ w.r.t.\
$\T$ cannot be solved using monotone circuits of polynomial size. Finally, we
show that existence of a rewriting for $\Q$ and $\T$ contradicts Item 1 of
Theorem \ref{th:datalog-inexpressibility}.

\section{Computing Rewritings via Resolution}\label{sec:DatalogRewritings}

Theorem \ref{th:non-Rewritable-General-ELU} is rather discouraging since it
applies to one of the simplest non-Horn languages. The theorem's proof,
however, relies on a specific TBox $\T$ that encodes a hard problem (i.e.,
non-3-colorability) that is not solvable by monotone circuits of polynomial
size. One can expect that non-Horn TBoxes used in practice do not encode such
hard problems, and so it might be possible to rewrite such TBoxes into
datalog.

\begin{table}[t]
\centering
\caption{Example TBox $\TEX$}\label{tab:example-TBox}
\begin{tabular}{cc}
    \hline
    $\gamma_1$  & $\Student  \sqsubseteq \Grad \sqcup \Undergrad$ \\
    $\gamma_2$  & $\Course \sqsubseteq \GradCo \sqcup \UndergradCo$ \\
    \hline
    $\gamma_3$  & $\PHD \sqsubseteq \exists \takes.\PHDco$ \\
    \hline
    $\gamma_4$  & $\PHDco \sqsubseteq \GradCo$ \\
    $\gamma_5$  & $\exists \takes.\GradCo \sqsubseteq \Grad$ \\
    $\gamma_6$  & $\Undergrad \sqcap \exists \takes.\GradCo \sqsubseteq \bot$ \\
    \hline
\end{tabular} 
\end{table} 

We illustrate this intuition using the TBox $\TEX$ shown in Table
\ref{tab:example-TBox}. Axioms $\gamma_4$--$\gamma_6$ correspond to datalog
rules, whereas axioms $\gamma_1$--$\gamma_3$ represent disjunctive and
existentially quantified knowledge and thus do not correspond to datalog
rules. We will show that $\TEX$ can, in fact, be rewritten into datalog using
a generic three-step method that takes a normalised $\SHI$-TBox $\T$ and
proceeds as follows.
\begin{enumerate}
    \item[\textbf{S1}] Eliminate the transitivity axioms from $\T$ by
    transforming $\T$ into an $\ALCHI$-TBox $\Omega_\T$ and a set of datalog
    rules $\Xi_\T$ such that facts entailed by ${\T \cup \A}$ and ${\Omega_\T
    \cup \Xi_\T(\A)}$ coincide for each ABox $\A$. This step extends the known
    technique to make it complete for facts with roles that have transitive
    subroles in $\T$.

    \item[\textbf{S2}] Apply the algorithm by \citeA{hms07reasoning} to
    transform $\Omega_\T$ into a disjunctive datalog program $\DD(\Omega_\T)$.

    \item[\textbf{S3}] Transform $\DD(\Omega_\T)$ into a set of datalog rules
    $\Phorn$ using a variant of the knowledge compilation techniques by
    \citeA{selman1996knowledge} and \citeA{DBLP:journals/ai/Val05}.
\end{enumerate}
Step \textbf{S3} may not terminate for an arbitrary $\SHI$-TBox $\T$; however,
if it terminates (i.e., if $\Phorn$ is finite), then ${\Phorn \cup \Xi_\T}$ is
a rewriting of $\T$. Furthermore, in Section \ref{sec:Termination} we show
that step \textbf{S3} always terminates if $\T$ is a $\Bool$-TBox. We thus
obtain what is, to the best of our knowledge, the first goal-oriented
rewriting algorithm for a practically-relevant non-Horn fragment of OWL 2 DL.

\subsection{Transitivity}\label{sec:Transitivity}

We first recapitulate the standard technique for eliminating transitivity
axioms from $\SHI$-TBoxes.

\begin{definition}\label{def:transitivity-standard}
    Let $\T$ be a normalised $\SHI$-TBox, and let $\Theta_\T$ be obtained from
    $\T$ by removing all transitivity axioms. If $\T$ is a $\Bool$-TBox, then
    let ${\Upsilon_\T = \Theta_\T}$; otherwise, let $\Upsilon_\T$ be the
    extension of $\Theta_\T$ with axioms
    \begin{displaymath}
        \exists R.A \sqsubseteq \transatom{B}{R}    \qquad  \exists R.\transatom{B}{R} \sqsubseteq \transatom{B}{R} \qquad  \transatom{B}{R} \sqsubseteq B
    \end{displaymath}
    for each axiom ${\exists S.A \sqsubseteq B \in \T}$ and each transitive
    role $R$ in $\T$ such that ${R \roleh{\T} S}$, where $\transatom{B}{R}$ is
    a fresh atomic concept unique for $B$ and $R$.
\end{definition}

This encoding preserves entailment of all facts of the form $C(c)$ and
$U(c,d)$ if $U$ has no transitive subroles: this was proved by
\citeA{Artale09thedl-lite} for $\Bool$, and by
\citeA{DBLP:conf/dlog/Simancik12} for $\SHI$. Example \ref{ex:transitivity},
however, shows that the encoding is incomplete if $U$ has transitive subroles.

\begin{example}\label{ex:transitivity}
Let $\T$ be the TBox below, and let ${\A = \{ A(a) \}}$.
\begin{displaymath}
    A \sqsubseteq \exists S.B \qquad S \sqsubseteq R \qquad S \sqsubseteq R^- \qquad \Tra{R}
\end{displaymath}
Then, ${\Upsilon_\T = \T \setminus \{ \Tra{R} \}}$, and one can easily verify
that ${\T \cup \A \models R(a,a)}$, but ${\Upsilon_\T \cup \A \not\models
R(a,a)}$. Note, however, that the missing inference can be recovered by
extending $\Upsilon_\T$ with the axiom ${A \sqsubseteq \eself{R}}$, which is a
consequence of $\T$.
\end{example}

The intuitions from Example \ref{ex:transitivity} are formalised in
Definition~\ref{def:transitivity}. Roughly speaking, we transform the
transitivity and role inclusion axioms in $\T$ into a datalog program
$\Xi_\T$, which we apply to $\A$ `first'---that is, we compute $\Xi_\T(\A)$
independently from any GCIs. To recoup the remaining consequences of the form
$R(a,a)$, we extend $\Upsilon_\T$ with sufficiently many axioms of the form
${A \sqsubseteq \eself{R}}$ that are entailed by $\T$; this is possible since
we assume that $\T$ is normalised.

\begin{definition}\label{def:transitivity}
    Let $\T$ be a normalised $\SHI$-TBox. Then, $\Omega_\T$ is the TBox
    obtained by extending $\Upsilon_\T$ with an axiom ${A \sqsubseteq
    \eself{R}}$ for each atomic concept $A$ and each atomic role $R$ such that
    $R$ is transitive in $\T$, and ${A \sqsubseteq \exists S.B \in \T}$ for
    some concept $B$ and role $S$ with ${S \roleh{\T} R}$ and ${S \roleh{\T}
    R^-}$. Furthermore, $\Xi_\T$ is the set of datalog rules corresponding to
    the role inclusion and transitivity axioms in $\T$.
\end{definition}

\begin{restatable}{theorem}{transitivity}\label{th:transitivity}
    Let $\T$ be a normalised $\SHI$-TBox, let $\A$ be an ABox, and let
    $\alpha$ be a fact. Then, ${\T \cup \A \models \alpha}$ if and only if
    ${\Omega_\T \cup \Xi_\T(\A) \models \alpha}$.
\end{restatable}

Note that, if $\T$ is normalised, so is $\Omega_\T$. Furthermore, to ensure
decidability, roles involving transitive subroles are not allowed occur in
$\T$ in number restrictions, and so Theorem \ref{th:transitivity} holds even
if $\T$ is a $\mathcal{SHOIQ}$-TBox.

\subsection{From DLs to Disjunctive Datalog}\label{sec:DLtoDD}

Step \textbf{S2} of our rewriting algorithm uses the technique by
\citeA{hms07reasoning} for transforming an $\ALCHI$-TBox $\T$ into a
disjunctive datalog program $\DD(\T)$ such that, for each ABox $\A$, the facts
entailed by ${\T \cup \A}$ and ${\DD(\T) \cup \A}$ coincide. By eliminating
the existential quantifiers in $\T$, one thus reduces a reasoning problem in
${\T \cup \A}$ to a reasoning problem in ${\DD(\T) \cup \A}$. The following
definition summarises the properties of the programs produced by the
transformation.

\begin{definition}\label{def:NMP}
    A disjunctive datalog program $\P$ is \emph{nearly-monadic} if its rules
    can be partitioned into two disjoint sets, $\Pmon$ and $\Prol$, such that
    \begin{enumerate}
        \item\label{cond:mon} each rule ${r \in \Pmon}$ mentions only unary
        and binary predicates and each atom in the head of $r$ is of the form
        $A(z)$ or $R(z,z)$ for some variable $z$, and
    
        \item\label{cond:rol} each rule ${r \in \Prol}$ is of the form
        ${R(x,y) \rightarrow S(x,y)}$ or ${R(x,y) \rightarrow S(y,x)}$.
    \end{enumerate}
    
    A disjunctive rule $r$ is \emph{simple} if there exists a variable $x$
    such that each atom in the body of $r$ is of the form $A_i(x)$,
    $R_i(x,x)$, $S_i(x,y_i)$, or $T_i(y_i,x)$, each atom in the head of $r$ is
    of the form $U_i(x,x)$ or $B_i(x)$, and each variable $y_i$ occurs in $r$
    at most once. Furthermore, a nearly-monadic program $\P$ is \emph{simple}
    if each rule in $\Pmon$ is simple.
\end{definition}

Theorem \ref{th:DLtoDD} follows mainly from the results by
\citeA{hms07reasoning}; we just argue that concepts $\eself{R}$ do not affect
the algorithm, and that $\DD(\T)$ satisfies property 1.

\begin{restatable}{theorem}{DLtoDD}\label{th:DLtoDD}
    For $\T$ a normalised $\ALCHI$-TBox, $\DD(\T)$ satisfies the following:
    \begin{enumerate}
        \item program $\DD(\T)$ is nearly-monadic; furthermore, if $\T$ is a
        $\Bool$-TBox, then $\DD(\T)$ is also simple;
        
        \item ${\T \models \DD(\T)}$; and
    
        \item ${\cert{\Q}{\T}{\A} = \cert{\Q}{\DD(\T)}{\A}}$ for each ABox
        $\A$ and each ground query $\Q$.
    \end{enumerate}
\end{restatable}

\begin{table}[t]
\centering
\caption{Example Disjunctive Program $\DD(\TEX)$}\label{tab:example-DD}
\begin{tabular}{cc}
    \hline
    $C_1$    & $\neg \Student(x) \vee \Grad(x) \vee \Undergrad(x)$ \\
    $C_2$    & $\neg \Course(x) \vee \GradCo(x) \vee \UndergradCo(x)$ \\
    \hline
    $C_3$    & $\neg \PHD(x) \vee \Grad(x)$ \\
    \hline
    $C_4$    & $\neg \PHDco(x) \vee \GradCo(x)$ \\
    $C_5$    & $\neg \takes(x,y) \vee \neg \GradCo(y) \vee \Grad(x)$ \\
    $C_6$    & $\neg \Undergrad(x) \vee \neg \takes(x,y) \vee \neg \GradCo(y)$ \\
    \hline
\end{tabular}  
\end{table} 

\begin{example}
When applied to the TBox $\TEX$ in Table \ref{tab:example-TBox}, this
algorithm produces the disjunctive program $\DD(\TEX)$ shown (as clauses) in
Table \ref{tab:example-DD}. In particular, axiom $\gamma_3$ is eliminated
since it contains an existential quantifier, but its effects are compensated
by clause $C_3$. Clauses $C_1$--$C_2$ and $C_4$--$C_6$ are obtained from
axioms $\gamma_1$--$\gamma_2$ and $\gamma_4$--$\gamma_6$, respectively.
\end{example}

\subsection{From Disjunctive Datalog to Datalog} \label{sec:DDtoD}

Step \textbf{S3} of our rewriting algorithm attempts to transform the
disjunctive program obtained in Step \textbf{S2} into a datalog program such
that, for each ABox $\A$, the two programs entail the same facts. This is
achieved using known knowledge compilation techniques, which we survey next.

\subsubsection{Resolution-Based Knowledge Compilation}

In their seminal paper, \citeA{selman1996knowledge} proposed an algorithm for
compiling a set of propositional clauses $\S$ into a set of Horn clauses
$\Shorn$ such that the Horn consequences of $\S$ and $\Shorn$ coincide.
Subsequently, \citeA{DBLP:journals/ai/Val05} generalised this algorithm to the
case when $\S$ contains first-order clauses, but without any termination
guarantees; Procedure~\ref{alg:Hornification} paraphrases this algorithm. The
algorithm applies to $\S$ binary resolution and positive factoring from
resolution theorem proving, and it keeps only the consequences that are not
redundant according to Definition \ref{def:relevant-consequence}. Unlike
standard resolution, the algorithm maintains two sets $\Shorn$ and $\Snhorn$
of Horn and non-Horn clauses, respectively; furthermore, the algorithm never
resolves two Horn clauses.

\begin{algorithm}[t!]
\caption{$\mathsf{Compile}$-$\mathsf{Horn}$}\label{alg:Hornification}
\begin{footnotesize}
    \textbf{Input:}
    $\S$: set of clauses \\
    \textbf{Output:}
    $\S_{H}$: set of Horn clauses
    \begin{algorithmic}[1]
        \State $\Shorn := \{ C \in \S \mid C \text{ is a Horn clause and not a tautology}\}$
        \State $\Snhorn := \{ C \in \S \mid C \text{ is a non-Horn clause and not a tautology}\}$
        \Repeat
            \State Compute all relevant consequences of $\tuple{\Shorn,\Snhorn}$
            \For{\textbf{each} relevant consequence $C$ of $\tuple{\Shorn,\Snhorn}$}
                \State Delete from $\Shorn$ and $\Snhorn$ all clauses $\theta$-subsumed by $C$
                \State \textbf{if} $C$ is Horn \textbf{then} $\Shorn := \Shorn \cup \{ C \}$
                \State \textbf{else} $\Snhorn := \Snhorn \cup \{ C \}$
            \EndFor
        \Until{there is no relevant consequence of $\tuple{\Shorn,\Snhorn}$}
        \State \Return $\Shorn$
    \end{algorithmic}
\end{footnotesize}
\end{algorithm}

\begin{definition}\label{def:relevant-consequence}
    Let $\Shorn$ and $\Snhorn$ be sets of Horn and non-Horn clauses,
    respectively. A clause $C$ is a \emph{relevant consequence} of
    $\tuple{\Shorn,\Snhorn}$ if
    \begin{itemize}
        \item $C$ is not redundant in ${\Shorn \cup \Snhorn}$, and
    
        \item $C$ is a factor of a clause ${C_1 \in \Snhorn}$, or a resolvent
        of clauses ${C_1 \in \Snhorn}$ and ${C_2 \in \Snhorn \cup \Shorn}$.
    \end{itemize}
\end{definition} 

Theorem~\ref{th:Del-Val} recapitulates the algorithm's properties. It
essentially shows that, even if the algorithm never terminates, each Horn
consequence of $\S$ will at some point during algorithm's execution become
entailed by the set of Horn clauses $\Shorn$ computed by the algorithm. The
theorem was proved by showing that each resolution proof of a consequence of
$\S$ can be transformed to `postpone' all resolution steps between two Horn
clauses until the end; thus, one can `precompute' set $\Shorn$ of all
consequences of $\S$ derivable using a non-Horn clause.

\begin{restatable}{theorem}{DelVal}(\cite{DBLP:journals/ai/Val05})\label{th:Del-Val}
    Let $\S$ be a set of clauses, and let $C$ be a Horn clause such that ${\S
    \models C}$, and assume that Procedure \ref{alg:Hornification} is applied
    to $\S$. Then, after some finite number of iterations of the loop in lines
    3--9, we have ${\Shorn \models C}$.
\end{restatable} 

\subsubsection{ABox-Independent Compilation}

Compiling knowledge into Horn clauses and computing datalog rewritings are
similar in spirit: both transform one theory into another while ensuring that
the two theories are indistinguishable w.r.t.\ a certain class of queries.
There is, however, an important difference: given a disjunctive program $\P$
and a fixed ABox $\A$, one could apply Procedure \ref{alg:Hornification} to
${\S = \P \cup \A}$ to obtain a datalog program $\Shorn$, but such $\Shorn$
would not necessarily be independent from the specific ABox $\A$. In contrast,
a rewriting of $\P$ is a datalog program $\Phorn$ that can be freely combined
with an arbitrary ABox $\A$. We next show that a program $\Phorn$ satisfying
the latter requirement can be obtained by applying Procedure
\ref{alg:Hornification} to $\P$ only.

Towards this goal, we generalise Theorem \ref{th:Del-Val} and show that, when
applied to an arbitrary set of first-order clauses $\N$, Procedure
\ref{alg:Hornification} computes a set of Horn clauses $\Nhorn$ such that the
Horn consequences of ${\N \cup \A}$ and ${\Nhorn \cup \A}$ coincide for an
arbitrary ABox $\A$. Intuitively, this shows that, when Procedure
\ref{alg:Hornification} is applied to ${\S = \N \cup \A}$, all inferences
involving facts in $\A$ can be `moved' to end of derivations.

\begin{restatable}{theorem}{postponeABox}\label{th:postpone-ABox}
    Let $\N$ be a set of clauses, let $\A$ be an ABox, let $C$ be a Horn
    clause such that ${\N \cup \A \models C}$, and assume that Procedure
    \ref{alg:Hornification} is applied to $\N$. Then, after some finite number
    of iterations of the loop in lines 3--9, we have ${\Nhorn \cup \A \models
    C}$.
\end{restatable}

\subsubsection{Rewriting Nearly-Monadic Disjunctive Programs}

The final obstacle to obtaining a datalog rewriting of a $\SHI$-TBox $\T$ is
due to Theorem \ref{th:transitivity}: the rules in $\Xi_\T$ should be applied
`before' $\Omega_\T$. While this allows us to transform $\Omega_\T$ into ${\P
= \DD(\Omega_\T)}$ and $\Phorn$ without taking $\Xi_\T$ into account, this
also means that Theorems \ref{th:transitivity}, \ref{th:DLtoDD}, and
\ref{th:postpone-ABox} only imply that the facts entailed by ${\T \cup \A}$
and ${\Phorn \cup \Xi_\T(\A)}$ coincide. To obtain a `true' rewriting, we show
in Lemma \ref{lem:compile-NMP} that program $\Phorn$ is nearly-monadic. We use
this observation in Theorem \ref{th:SHI-rewriting} to show that each binary
fact obtained by applying $\Phorn$ to $\Xi_\T(\A)$ is of the form $R(c,c)$,
and so it cannot `fire' the rules in $\Xi_\T$; hence, ${\Phorn \cup \Xi_\T}$
is a rewriting of $\T$.

\begin{restatable}{lemma}{compileNMP}\label{lem:compile-NMP} Let $\P$ be a
nearly-monadic program, and assume that Procedure \ref{alg:Hornification}
terminates when applied to $\P$ and returns $\Phorn$. Then, $\Phorn$ is a
nearly-monadic datalog program. \end{restatable}

\begin{restatable}{theorem}{SHIrewriting}\label{th:SHI-rewriting}
    Let ${\P = \DD(\Omega_\T)}$ for $\T$ an $\SHI$-TBox. If, when applied to
    $\P$, Procedure \ref{alg:Hornification} terminates and returns $\Phorn$,
    then ${\Phorn \cup \Xi_\T}$ is a rewriting of $\T$.
\end{restatable}

Please note that our algorithm (just like all rewriting algorithms we are
aware of) computes rewritings using a sound inference system and thus always
produces strong rewritings.

\begin{example}\label{ex:running-example}
When applied to the program ${\P = \DD(\TEX)}$ from Table
\ref{tab:example-DD}, Procedure \ref{alg:Hornification} resolves $C_2$ and
$C_5$ to derive \eqref{eq:clause-derived-1}, $C_2$ and $C_6$ to derive
\eqref{eq:clause-derived-2}, and $C_1$ and $C_6$ to derive
\eqref{eq:clause-derived-3}.
\begingroup\small
\begin{align}
    \label{eq:clause-derived-1} \neg \takes(x,y) \vee \neg \Course(y) \vee \Grad(x) \vee \UndergradCo(y) \\
    \label{eq:clause-derived-2} \neg \takes(x,y) \vee \neg \Undergrad(x) \vee \neg \Course(y) \vee \UndergradCo(y) \\
    \label{eq:clause-derived-3} \neg \takes(x,y) \vee \neg \Student(x) \vee \neg \GradCo(y) \vee \Grad(x)
\end{align}
\endgroup
Resolving \eqref{eq:clause-derived-2} and $C_1$, and
\eqref{eq:clause-derived-3} and $C_2$ produces redundant clauses, after which
the procedure terminates and returns the set $\Phorn$ consisting of clauses
$C_3$--$C_6$, \eqref{eq:clause-derived-2}, and \eqref{eq:clause-derived-3}. By
Theorem \ref{th:SHI-rewriting}, $\Phorn$ is a strong rewriting of $\TEX$.
\end{example}

\subsection{Termination}\label{sec:Termination}

Procedure \ref{alg:Hornification} is not a semi-decision procedure for either
strong non-rewritability (cf.\ Example \ref{ex:non-termination-1}) or strong
rewritability (cf.\ Example \ref{ex:non-termination-2}) of nearly-monadic
programs.

\begin{example}\label{ex:non-termination-1}  Let $\P$ be defined
as follows.
\begin{align}
    G(x) \vee B(x)                                  & \label{eq:non-term-1} \\
    B(x_1) \vee \neg E(x_1,x_0) \vee \neg G(x_0)    & \label{eq:non-term-2} \\
    G(x_1) \vee \neg E(x_1,x_0) \vee \neg B(x_0)    & \label{eq:non-term-3}
\end{align}
Clauses \eqref{eq:non-term-2} and \eqref{eq:non-term-3} are mutually
recursive, but they are also Horn, so Procedure \ref{alg:Hornification} never
resolves them directly.

Clauses \eqref{eq:non-term-2} and \eqref{eq:non-term-3}, however, can interact
through clause \eqref{eq:non-term-1}. Resolving \eqref{eq:non-term-1} and
\eqref{eq:non-term-2} on ${\neg G(x_0)}$ produces \eqref{eq:non-term-4}; and
resolving \eqref{eq:non-term-3} and \eqref{eq:non-term-4} on $B(x_1)$ produces
\eqref{eq:non-term-5}. By further resolving \eqref{eq:non-term-5}
alternatively with \eqref{eq:non-term-2} and \eqref{eq:non-term-3}, we obtain
\eqref{eq:non-term-6} for each even $n$. By resolving \eqref{eq:non-term-3}
and \eqref{eq:non-term-6} on $B(x_0)$, we obtain \eqref{eq:non-term-7}.
Finally, by factoring \eqref{eq:non-term-7}, we obtain \eqref{eq:non-term-8}
for each even $n$.
\begin{align}
    B(x_1) \vee \neg E(x_1,x_0) \vee B(x_0)                                                 & \label{eq:non-term-4} \\
    G(x_2) \vee \neg E(x_2,x_1) \vee \neg E(x_1,x_0) \vee B(x_0)                            & \label{eq:non-term-5} \\
    G(x_n) \vee [ \bigvee_{i=1}^n \neg E(x_i,x_{i-1}) ] \vee B(x_0)                         & \label{eq:non-term-6} \\
    G(x_n) \vee [ \bigvee_{i=1}^n \neg E(x_i,x_{i-1}) ] \vee G(x_1') \vee \neg E(x_1',x_0)  & \label{eq:non-term-7} \\
    G(x_n) \vee \neg E(x_n,x_0) \vee [ \bigvee_{i=1}^n \neg E(x_i,x_{i-1}) ]                & \label{eq:non-term-8}
\end{align}
Procedure \ref{alg:Hornification} thus derives on $\P$ an infinite set of Horn
clauses, and Theorem \ref{th:negative-model-preserving} shows that no strong
rewriting of $\P$ exists.
\end{example}

\begin{example}\label{ex:non-termination-2}
Let $\P$ be defined as follows.
\begin{align}
    B_1(x_0) \vee B_2(x_0) \vee \neg A(x_0)         & \label{eq:non-term-ex2-1} \\
    A(x_1) \vee \neg E(x_1,x_0) \vee \neg B_1(x_0)  & \label{eq:non-term-ex2-2} \\
    A(x_1) \vee \neg E(x_1,x_0) \vee \neg B_2(x_0)  & \label{eq:non-term-ex2-3}
\end{align}
When applied to $\P$, Procedure \ref{alg:Hornification} will eventually
compute infinitely many clauses $C_n$ of the following form:
\begin{equation*}
    C_n = A(x_n) \vee [ \bigvee_{i=1}^n \neg E(x_i,x_{i-1})] \vee \neg A(x_0)
\end{equation*}
However, for each ${n > 1}$, clause $C_n$ is a logical consequence of clause
$C_1$, so the program consisting of clauses \eqref{eq:non-term-ex2-1},
\eqref{eq:non-term-ex2-2}, and $C_1$ is a strong rewriting of $\P$.
\end{example}

Example \ref{ex:condensation} demonstrates another problem that can arise even
if $\P$ is nearly-monadic and simple.

\begin{example}\label{ex:condensation}
Let $\P$ be the following program:
\begin{align}
    \neg R(x,y) \vee A(x)                           & \label{eq:condensation-1} \\
    \neg R(x,y) \vee B(x)                           & \label{eq:condensation-2} \\
    \neg A(x) \vee \neg B(x) \vee C(x) \vee D(x)    & \label{eq:condensation-3}
\end{align}
Now resolving \eqref{eq:condensation-1} and \eqref{eq:condensation-3} produces
\eqref{eq:condensation-4}; and resolving \eqref{eq:condensation-2} and
\eqref{eq:condensation-4} produces \eqref{eq:condensation-5}.
\begin{align}
    \neg R(x,y) \vee \neg B(x) \vee C(x) \vee D(x)          & \label{eq:condensation-4} \\
    \neg R(x,y_1) \vee \neg R(x,y_2) \vee C(x) \vee D(x)    & \label{eq:condensation-5}
\end{align}
Clause \eqref{eq:condensation-5} contains more variables than clauses
\eqref{eq:condensation-1} and \eqref{eq:condensation-2}, which makes bounding
the clause size difficult.
\end{example}

Notwithstanding Example \ref{ex:condensation}, we believe one can prove that
Procedure \ref{alg:Hornification} terminates if $\P$ is nearly-monadic and
simple. However, apart from making the termination proof more involved,
deriving clauses such as \eqref{eq:condensation-5} is clearly inefficient. We
therefore extend Procedure \ref{alg:Hornification} with the condensation
simplification rule, which eliminates redundant literals in clauses such as
\eqref{eq:condensation-5}. A \emph{condensation} of a clause $C$ is a clause
$D$ with the least number of literals such that ${D \subseteq C}$ and $C$
subsumes $D$. A condensation of $C$ is unique up to variable renaming, so we
usually speak of \emph{the} condensation of $C$. We next show that Theorems
\ref{th:Del-Val} and \ref{th:postpone-ABox} hold even with condensation.

\begin{restatable}{lemma}{procwithcondensation}\label{lem:proc-with-condensation}
    Theorems \ref{th:Del-Val} and \ref{th:postpone-ABox} hold if
    Procedure~\ref{alg:Hornification} is modified so that, after line 5, $C$
    is replaced with its condensation.
\end{restatable}

One can prove that all relevant consequences of nearly-monadic and simple
clauses are also nearly-monadic and simple, so by using condensation to remove
redundant literals, we obtain Lemma \ref{lem:termination-simple}, which
clearly implies Theorem \ref{th:rewriting-Bool}.

\begin{restatable}{lemma}{terminationsimple}\label{lem:termination-simple}
    If used with condensation, Procedure \ref{alg:Hornification} terminates
    when applied to a simple nearly-monadic program $\P$.
\end{restatable}

\begin{restatable}{theorem}{rewritingBool}\label{th:rewriting-Bool}
    Let ${\P = \DD(\Omega_\T)}$ for $\T$ a $\Bool$-TBox. Procedure
    \ref{alg:Hornification} with condensation terminates when applied to $\P$
    and returns $\Phorn$; furthermore, ${\Phorn \cup \Xi_\T}$ is a rewriting
    of $\T$.
\end{restatable}

We thus obtain a tractable (w.r.t.\ data complexity) procedure for answering
queries over $\Bool$-TBoxes. Furthermore, given a ground query $\Q$ and a
nearly-monadic and simple program $\Phorn$ obtained by Theorem
\ref{th:rewriting-Bool}, it should be possible to match the $\NLogSpace$ lower
complexity bound by \citeA{Artale09thedl-lite} as follows. First, one should
apply backward chaining to $\Q$ and $\Phorn$ to compute a UCQ $\Q'$ such that
${\cert{\Q}{\Phorn}{\Xi_\T(\A)} = \cert{\Q'}{\emptyset}{\Xi_\T(\A)}}$; since
all nearly-monadic rules in $\Phorn$ are simple, it should be possible to show
that such `unfolding' always terminates. Second, one should transform $\Xi_\T$
into an equivalent piecewise-linear datalog program $\Xi_\T'$. Although these
transformations should be relatively straightforward, a formal proof would
require additional machinery and is thus left for future work.

\section{Limits to Strong Rewritability} \label{sec:LimitsStrong}

We next show that strong rewritings may not exist for rather simple non-Horn
$\ELU$-TBoxes that are rewritable in general. This is interesting because it
shows that an algorithm capable of rewriting a larger class of TBoxes
necessarily must depart from the common approaches based on sound inferences.

\begin{restatable}{theorem}{negativemodelpreserving}\label{th:negative-model-preserving}
    The $\ELU$-TBox $\T$ corresponding to the program $\P$ from Example
    \ref{ex:non-termination-1} and the ground CQ ${\Q = G(x_1)}$ are
    $\Q$-rewritable, but not strongly $\Q$-rewritable.
\end{restatable}

The proof of Theorem \ref{th:negative-model-preserving} proceeds as follows.
First, we show that, for each ABox $\A$ encoding a directed graph, we have
${\cert{\Q}{\T}{\A} \neq \emptyset}$ iff the graph contains a pair of vertices
reachable by both an even and an odd number of edges. Second, we show that
latter property can be decided using a datalog program that uses new relations
not occurring in $\T$. Third, we construct an infinite set of rules $\R$
entailed by each strong rewriting of $\T$. Fourth, we show that ${\R'
\not\models \R}$ holds for each finite datalog program $\R'$ such that ${\T
\models \R'}$.
 
Since our procedure from Section \ref{sec:DatalogRewritings} produces only
strong rewritings, it cannot terminate on a TBox that has no strong
rewritings. This is illustrated in Example \ref{ex:non-termination-1}, which
shows that Procedure \ref{alg:Hornification} does not terminate when applied
to (the clausification of) the TBox from Theorem
\ref{th:negative-model-preserving}.

\section{Outlook}

Our work opens many possibilities for future research. On the theoretical
side, we will investigate whether one can decide existence of a strong
rewriting for a given $\SHI$-TBox $\T$, and to modify Procedure
\ref{alg:Hornification} so that termination is guaranteed.
\citeA{CarstenPODS2013} recently showed that rewritability of unary ground
queries over $\ALC$-TBoxes is decidable; however, their result does not
consider strong rewritability or binary ground queries. On the practical side,
we will investigate whether Procedure \ref{alg:Hornification} can be modified
to use ordered resolution instead of unrestricted resolution. We will also
implement our technique and evaluate its applicability.

\clearpage
\bibliographystyle{named}
\bibliography{ijcai13}

\ifdraft{
    \clearpage
    \onecolumn
    \appendix
    
    \section{Proofs for Section \ref{sec:NegativeResults}}\label{sec:Proofs-NegativeResults}

Before presenting the proof of Theorem \ref{th:non-Rewritable-General-ELU}, we
recapitulate the definition of \emph{monotone polynomial projections}, which
are frequently used to transfer bounds on the circuit size from one family of
monotone Boolean functions to another. Let $f$ be a monotone Boolean function
with inputs $\vec{x}$, and let $g$ be a monotone Boolean function with inputs
$\vec{y}$. Then, $f$ is a \emph{monotone projection} of $g$ if a mapping
${\rho : \vec{y} \rightarrow \{ \false,\true \} \cup \vec{x}}$ exists such
that ${f(\vec{x}) = g(\rho(\vec{y}))}$ for each value of $\vec{x}$. Given such
a mapping $\rho$, a monotone circuit that computes $g(\vec{y})$ can be
transformed to a monotone circuit that computes $f(\vec{x})$ by replacing each
input ${y_i \in \vec{y}}$ with $\rho(y_i)$. Furthermore, a family of Boolean
functions ${\{ f_n \}}$ is a \emph{polynomial monotone projection} of a family
${\{ g_k \}}$ if a polynomial $p(n)$ exists such that each $f_n$ is a monotone
projection of some $g_k$ with ${k \leq p(n)}$; if that is the case and the
family of functions ${\{ g_k \}}$ can be realised by a family of monotone
circuits of polynomial size, then so can ${\{ f_n \}}$.

\nonRewritableGeneralELU*

\begin{proof}
Let $\T$ be the following acyclic $\ELU$-TBox:
\begin{displaymath}
\begin{array}{r@{\;}l@{\qquad}r@{\;}l}
    F_R & \equiv R \sqcap \exists \text{edge}.R     & F_B           & \equiv B \sqcap \exists \text{edge}.B \\[0.5ex]
    F_G & \equiv G \sqcap \exists \text{edge}.G     & F             & \equiv F_R \sqcup F_B \sqcup F_G \\[0.5ex]
    V   & \sqsubseteq R \sqcup G \sqcup B           & \mathit{NC}   & \equiv \exists \text{vertex}.F \\
\end{array}
\end{displaymath}
Furthermore, let $v$ be a fixed individual, and let ${\Q = \mathit{NC}(v)}$.
We next represent the problem of answering $\Q$ over $\T$ and an arbitrary
input ABox $\A$ using a family of monotone functions ${\{ g_{n(u)} \}}$. The
\emph{input size} of $\A$ is the number $u$ of individuals occurring in $\A$
different from the fixed individual $v$; we assume that these individuals are
labelled ${a_1, \ldots, a_u}$. Furthermore, to unify the notation, let
${a_{u+1} = v}$. Using the signature of $\T$, one can then construct at most
${n(u) = 2(u+1)^2 + 9(u+1)}$ assertions; hence, we encode $\A$ using $n(u)$
bits $y_{i,j}^\text{edge}$, $y_{i,j}^\text{vertex}$, and $y_i^A$ as follows:
\begin{itemize}
    \item for each ${R \in \{ \text{edge},\text{vertex} \}}$ and ${1 \leq i,j
    \leq u}$, bit $y_{i,j}^R$ is $\true$ if and only if ${R(a_i,a_j) \in \A}$;
    and

    \item for each ${A \in \{ R,G,B,F_R,F_B,F_G,F,V,\mathit{NC} \}}$, bit
    $y_i^A$ is $\true$ if and only if ${A(a_i) \in \A}$.
\end{itemize}
The family of Boolean functions ${\{ g_{n(u)} \}}$ is defined such that, given
a vector of bits $\vec{y}$ encoding an ABox $\A$ of input size $u$, we have
${g_{n(u)}(\vec{y}) = \true}$ if and only if ${\T \cup \A \models \Q}$. Since
first-order logic is monotonic, each $g_{n(u)}$ is clearly monotone.

Let ${\{ f_{m(s)} \}}$ be the family of monotone Boolean functions associated
with non-3-colorability as defined in Section \ref{sec:NegativeResults}. We
next show that ${\{ f_{m(s)} \}}$ is a monotone polynomial projection of ${\{
g_{n(u)} \}}$. To this end, we first show that, for each positive integer $s$,
function $f_{m(s)}$ is a monotone projection of $g_{n(s)}$. Let $\rho$ be the
following mapping, where $A$ is a placeholder for each concept from the
signature of $\T$ different from $V$:
\begin{displaymath}
\begin{array}{@{}r@{\;}l@{}}
    \rho(y_{i,j}^{\text{edge}}) = & \rho(y_{j,i}^{\text{edge}}) =
    \begin{cases}
        x_{i,j} & \text{for } 1 \leq i < j \leq s \\
        \false  & \text{otherwise} \\
    \end{cases} \\[2ex]
    \rho(y_{i,j}^{\text{vertex}}) = &
    \begin{cases}
        \true   & \text{for } i = s + 1 \text{ and } 1 \leq j \leq s \\
        \false  & \text{otherwise} \\
    \end{cases} \\[2.4ex]
    \rho(y_i)^V = &
    \begin{cases}
        \true   & \text{for } 1 \leq i \leq s \\
        \false  & \text{for } i = s+1 \\
    \end{cases} \\[2.2ex]
    \rho(y_i)^A =   & \false \quad \text{for } 1 \leq i \leq s+1
\end{array}
\end{displaymath}
We now show that ${f_{m(s)}(\vec{x}) = g_{n(s)}(\rho(\vec{y}))}$ for each
vector $\vec{x}$ of $m(s)$ bits. To this end, let $G$ be the undirected graph
associated with $\vec{x}$ containing nodes ${1, \ldots, s}$. It is
straightforward to check that $\rho(\vec{y})$ is then a vector of $n(s)$ bits
that encodes the ABox $\A_G$ with individuals ${a_1, \ldots, a_s, a_{s+1} =
v}$ containing the following assertions:
\begin{itemize}
    \item ${\text{edge}(a_i,a_j)}$ and ${\text{edge}(a_j,a_i)}$ for all ${1
    \leq i < j \leq s}$ such that $G$ contains an edge between $i$ and $j$,
    and

    \item assertions $V(a_i)$ and ${\text{vertex}(v,a_i)}$ for each ${1 \leq i
    \leq s}$.
\end{itemize}
Furthermore, it is routine to check that $G$ is non-3-colorable iff ${\T \cup
\A_G \models \Q}$; but then, by the definition of $f_{m(s)}$ and $g_{n(s)}$,
we have ${f_{m(s)}(\vec{x}) = g_{n(s)}(\rho(\vec{y}))}$, as required. Finally,
for ${p(z) = z^2}$, we clearly have ${n(s) \leq (m(s))^2}$. Thus, the family
of monotone functions ${\{ f_{m(s)} \}}$ is a monotone polynomial projection
of the family of monotone functions ${\{ g_{n(s)} \}}$.

The above observation, Item 2 of Theorem \ref{th:datalog-inexpressibility},
and the properties of monotone polynomial projections imply that the query
answering problem for $\Q$ and $\T$ cannot be solved using monotone circuits
of polynomial size. Now assume that a datalog program $\P$ exists that is a
$\Q$-rewriting of $\T$. By Item 1 of Theorem
\ref{th:datalog-inexpressibility}, answering $\Q$ over $\P$, and so the
problem of answering $\Q$ over $\T$ as well, can be solved using monotone
circuits of polynomial size, which is a contradiction.
\end{proof}

\section{Proof of Theorem \ref{th:transitivity}}\label{sec:Proofs-Transitivity}

\transitivity*

\begin{proof}
We prove the contrapositive: for each fact $\alpha$, we have ${\T \cup \A
\not\models \alpha}$ if and only if ${\Omega_\T \cup \Xi_\T(\A) \not\models
\alpha}$.

\medskip

($\Rightarrow$) It is routine to show that $\Upsilon_\T$ is a
model-conservative extension of $\T$ \cite{DBLP:conf/dlog/Simancik12}.
Furthermore, for all concepts $A$ and $B$, each atomic role $R$, and each role
$S$ such that ${A \sqsubseteq \exists S.B \in \T}$, ${S \roleh{\T} R}$, and
${S \roleh{\T} R^-}$, we clearly have ${\T \models A \sqsubseteq \eself{R}}$.
By these two properties, $\Omega_\T$ is a model-conservative extension of
$\T$. Finally, it is obvious that ${\T \models \Xi_\T}$. Now consider an
arbitrary fact $\alpha$ such that ${\T \cup \A \not\models \alpha}$. Then, an
interpretation $I$ exists such that ${I \models \T \cup \A}$ and ${I
\not\models \alpha}$. Since $\Omega_\T$ is a model-conservative extension of
$\T$ and ${\T \models \Xi_\T}$, an interpretation $J$ exists such that ${J
\models \Omega_\T}$ and ${J \models \Xi_\T(\A)}$; furthermore, since $\alpha$
does not use the symbols occurring in $\Omega_\T$ but not in $\T$, we also
have ${J \not\models \alpha}$. Thus, we have ${\Omega_\T \cup \Xi_\T(\A)
\not\models \alpha}$, as required.

\medskip

($\Leftarrow$) Consider an arbitrary fact $\alpha$ such that ${\Omega_\T \cup
\Xi_\T(\A) \not\models \alpha}$. Then, an interpretation ${I = (\Delta^I,
\cdot^I)}$ exists such that ${I \models \Omega_\T \cup \Xi_\T(\A)}$ and ${I
\not\models \alpha}$. Without loss of generality, we can assume that $I$ is of
a special tree shape, which we describe next. Let $N_\A$ be the set of
individuals occurring in $\A$, and let $N$ be the smallest set such that
${N_\A \subseteq N}$ and, if ${u \in N}$, then ${u.i \in N}$ for each
nonnegative integer $i$. Then, we can assume that $I$ satisfies all of the
following properties:
\begin{enumerate}
    \item\label{can:1} ${\Delta^I \subseteq N}$;

    \item\label{can:2} $c^I = c$ for each individual ${c \in N_\A}$;

    \item\label{can:3} for each atomic role $R$, each pair in $R^I$ is of the
    form $\tuple{s,s.i}$, $\tuple{s.i,s}$, or $\tuple{a,b}$ for ${s \in N}$
    and ${a,b \in N_\A}$;
    
    \item\label{can:4} for each pair ${\tuple{c,d} \in R^I}$ such that ${c,d
    \in N_\A}$, we have ${c = d}$ or ${R(c,d) \in \Xi_\T(\A)}$; and

    \item\label{can:5} for each atomic role $R$, each individual ${c \in
    N_\A}$ and each ${c.i \in N}$, if ${\{ \tuple{c,c.i}, \tuple{c.i,c} \}
    \subseteq R^I}$, then there exist concepts $A$ and $B$ and a role $S$ such
    that ${A \sqsubseteq \exists S.B \in \T}$, ${S \roleh{\T} R}$, ${S
    \roleh{\T} R^-}$, and ${c \in A^I}$.
\end{enumerate}
A model $I$ of $\Omega_\T$ satisfying properties \eqref{can:1}--\eqref{can:3}
can be obtained, for example, using the hypertableau calculus by
\citeA{msh09hypertableau}. Furthermore, if translated into first-order logic,
all role atoms in the consequent of an axiom in $\Omega_\T$ are of the form
$R(x,x)$, or they occur in formulae of the form ${\exists y. R(x,y) \wedge
\ldots}$; thus, the hypertableau calculus cannot derive an atom of the form
$R(a,b)$ with ${a \neq b}$, thus ensuring property \eqref{can:4}. Finally,
since $\Omega_\T$ is normalised, concepts of the form ${\exists S.B}$ occur in
$\Omega_\T$ only in axioms of the form ${A \sqsubseteq \exists S.B}$; but
then, the hypertableau calculus ensures that ${\tuple{c,c.i} \in S^I}$ or
${\tuple{c.i,c} \in S^I}$ only if ${c \in A^I}$; consequently, the only way
for ${\{ \tuple{c,c.i}, \tuple{c.i,c} \} \subseteq R^I}$ to hold is if
property \eqref{can:5} holds.

To complete the proof, we next construct an interpretation $J$ and show that
${J \models \T \cup \A}$ and ${J \not\models \alpha}$. In particular, let $J$
be the following interpretation defined inductively on the quasi-ordering
corresponding to relation $\roleh{\T}$:
\begin{itemize}
    \item ${\Delta^J = \Delta^I}$;

    \item ${c^J = c^I = c}$ for each individual ${c \in N_\A}$;

    \item ${A^J = A^I}$ for each atomic concept $A$;

    \item $R^J$ is the transitive closure of $R^I$ for each atomic role $R$
    that is transitive in $\T$; and

    \item ${R^J = R^I \cup \bigcup\limits_{S \roleh{\T} R \text{ and } R
    \not\roleh{\T} S} S^J}$ for each atomic role $R$ that is not transitive in
    $\T$.
\end{itemize}

If $\T$ does not contain concepts of the form $\eself{R}$, then ${J \models
\T}$ follows from the standard proofs of transitivity elimination in $\SHI$
\cite{DBLP:conf/dlog/Simancik12} and $\Bool$ \cite{Artale09thedl-lite};
furthermore, it is easy to see that the presence of atoms $\eself{R}$ requires
only minor changes to these proofs. Furthermore, since ${\A \subseteq
\Xi_\T(\A)}$, we clearly have ${J \models \A}$.

We are left to show that ${J \not\models \alpha}$. If $\alpha$ is of the form
$A(c)$, the claim follows from the proofs by \citeA{DBLP:conf/dlog/Simancik12}
and \citeA{Artale09thedl-lite}. Hence, assume that $\alpha$ is of the form
${\alpha = T(c,d)}$, and assume for the sake of contradiction that ${J \models
T(c,d)}$. Then, by the definition of $J$, there exist an atomic role $R$ and
${\{ u_0, u_1, \ldots, u_n \} \subseteq \Delta^I}$ such that $R$ is transitive
in $\T$, ${R \roleh{\T} T}$, ${c = u_0}$, ${d = u_n}$, and
${\tuple{u_{i-1},u_i} \in R^I}$ for each ${1 \leq i \leq n}$. We consider the
following two cases.
\begin{itemize}
    \item Assume that, for each ${0 \leq i \leq n}$, if ${u_i \in N_\A}$, then
    ${u_i = c}$. Then, we clearly have ${c = d}$. Since $I$ satisfies property
    \eqref{can:3}, some ${1 \leq i < n}$ exists such that $u_i$ is of the form
    $c.j$ for some $j$ and ${\{ \tuple{c,c.j}, \tuple{c.j,c} \} \subseteq
    R^I}$ holds. Furthermore, since $I$ satisfies property \eqref{can:5},
    concepts $A$ and $B$ and a role $S$ exist such that ${A \sqsubseteq
    \exists S.B \in \T}$, ${S \roleh{\T} R}$, ${S \roleh{\T} R^-}$, and ${c
    \in A^I}$. By Definition \ref{def:transitivity}, then ${A \sqsubseteq
    \eself{R} \in \Omega_\T}$, which implies ${\tuple{c,c} \in R^I}$. Finally,
    ${R \roleh{\T} T}$ implies ${R \roleh{\Omega_\T} T}$; hence, we have
    ${\tuple{c,c} \in T^I}$ as well, which contradicts our assumption that ${I
    \not\models \alpha}$.

    \item Assume that some ${1 \leq i \leq n}$ exists such that ${u_i \in
    N_\A}$ and ${u_i \neq c}$. We eliminate from the sequence ${u_0, u_1,
    \ldots, u_n}$ each subsequence ${u_{i+1}, \ldots, u_j}$ with ${0 \leq i <
    j \leq n}$ such that ${u_i \in N_\A}$, ${u_j \in N_\A}$, and ${u_k \in N
    \setminus N_\A}$ for each ${i < k < j}$; let ${v_0, \ldots, v_\ell}$ be
    the resulting sequence. Since $I$ satisfies property \eqref{can:3}, each
    eliminated subsequence satisfies ${u_i = u_j}$; hence, for each ${1 \leq i
    \leq \ell}$, we have ${\tuple{v_{i-1},v_i} \in R^I}$. Furthermore, since
    $u_i$ exists such that ${u_i \in N_\A}$ and ${u_i \neq c}$, we have ${\ell
    \geq 1}$, ${v_0 = c}$, and ${v_\ell = d}$. Finally, note that the above
    definition eliminates each subsequence ${u_i,u_{i+1}}$ such that ${u_i =
    u_{i+1}}$ (condition ${u_k \in N \setminus N_\A}$ for each ${i < k < j}$
    is then vacuously satisfied); therefore, sequence ${v_0, \ldots, v_\ell}$
    consists of distinct individuals in $N_\A$. But then, since $I$ satisfies
    property \eqref{can:4}, we have that ${R(v_{i-1},v_i) \in \Xi_\T(\A)}$ for
    each ${1 \leq i \leq \ell}$. Finally, by the definition of $\Xi_\T$, then
    $\Xi_\T(\A)$ contains ${R(v_0,v_\ell) = R(c,d)}$, and consequently
    ${T(c,d) \in \Xi_\T(\A)}$ as well. This, however, contradicts our
    assumption that ${I \not\models \alpha}$. \qedhere
\end{itemize}
\end{proof}

\section{Proofs for Section \ref{sec:DLtoDD}}\label{sec:Proofs-DLtoDD}

\DLtoDD*

\begin{proof}[Sketch]
The algorithm by \citeA{hms07reasoning} first translates $\T$ into a set of
skolemised clauses. An inspection of the algorithm reveals that, without
concepts of the form $\eself{R}$, each resulting clause is of one of the
following forms, where $R$ is an atomic role, $f$ is a function symbol, and
$A_{(i)}$, $B_{(i)}$, $C_{(i)}$, and $D_{(i)}$ are atomic concepts, $\top$, or
$\bot$:
\begin{align}
    \label{ctype:1} \neg A(x) \vee \underline{R(x,f(x))} \\
    \label{ctype:2} \neg A(x) \vee \underline{R(f(x),x)} \\
    \label{ctype:3} \underline{\neg R(x,y)} \vee S(x,y) \\
    \label{ctype:4} \underline{\neg R(x,y)} \vee S(y,x) \\
    \label{ctype:5} \neg A(x) \vee \underline{\neg R(x,y)} \vee \neg B(y) \vee C(x) \vee D(y) \\
    \label{ctype:6} \bigvee \neg A_i(x) \vee \bigvee \underline{\neg B_i(f(x))} \vee \bigvee C_i(x) \vee \bigvee \underline{D_i(f(x))} \\
    \label{ctype:7} \bigvee \underline{\neg A_i(x)} \vee \bigvee \underline{C_i(x)}
\end{align}
Furthermore, since $\T$ is normalised, axioms with concepts of the form
$\eself{R}$ are translated into clauses of the following form:
\begin{align}
    \label{ctype:8} \neg \underline{R(x,x)} \vee A(x) \\
    \label{ctype:9} \neg A(x) \vee \underline{R(x,x)}
\end{align}

The algorithm next saturates the resulting set of clauses by \emph{ordered}
resolution, which is parameterised by a carefully constructed literal ordering
and selection function; these parameters ensures that binary resolution and
positive factoring are performed only with literals that are underlined in
\eqref{ctype:1}--\eqref{ctype:9}. The selection function can be extended to
select atom $R(x,x)$ in each clause of type \eqref{ctype:8}; furthermore, the
ordering can be modified so that each atom $R(x,x)$ is larger than all atoms
$A(x)$, thus ensuring that only atom $R(x,x)$ participates in inferences with
clauses of type \eqref{ctype:9}. \citeA{hms07reasoning} show that each binary
resolution or positive factoring inference, when applied to clauses of type
\eqref{ctype:1}--\eqref{ctype:7}, produces a clause of type
\eqref{ctype:1}--\eqref{ctype:4} or \eqref{ctype:6}--\eqref{ctype:7}. This is
easily extended to clauses of type \eqref{ctype:8}--\eqref{ctype:9}:
\begin{itemize}
    \item a clause of type \eqref{ctype:8} cannot be resolved with any other
    clause;

    \item resolving a clause of type \eqref{ctype:9} with a clause of type
    \eqref{ctype:5} produces a clause of type \eqref{ctype:7}; and

    \item resolving a clause of type \eqref{ctype:9} with a clause of type
    \eqref{ctype:3} or \eqref{ctype:4} produces a clause of type
    \eqref{ctype:9}.
\end{itemize}

\citeA{hms07reasoning} then show that the disjunctive program $\DD(\T)$ can be
obtained as the set of all clauses after saturation of type
\eqref{ctype:3}--\eqref{ctype:5} and \eqref{ctype:7}. For the case when $\T$
contains atoms of the form $\eself{R}$, program $\DD(\T)$ should also include
clauses of type \eqref{ctype:8} and \eqref{ctype:9}, and the proof by
\citeA{hms07reasoning} applies without any problems. Furthermore, it is
straightforward to verify that $\DD(\T)$ is a nearly-monadic program.

Finally, if $\T$ is a $\Bool$-TBox, the only difference is that, in each
clause of type \eqref{ctype:5}, we have either ${A = \top}$ and ${C = \bot}$,
or ${B = \top}$ and ${D = \bot}$. Since saturation does not introduce clauses
of type \eqref{ctype:5}, program $\DD(\T)$ is clearly simple.
\end{proof}  
  
\section{Proofs for Section \ref{sec:DDtoD}}\label{sec:Proofs-DDtoD}

\postponeABox*

\begin{proof}
To prove our claim, we assume that Procedure \ref{alg:Hornification} is
applied to ${\S = \N \cup \A}$. Towards this goal, we associate with each
clause ${C \in \Shorn \cup \Snhorn}$ a set of facts $F_C$; for each such
$F_C$, let ${\neg F_C = \bigvee_{A \in F_C} \neg A}$. We define $F_C$
inductively on the applications of inference rules in Procedure
\ref{alg:Hornification}; furthermore, we show in parallel that, at any point
in time, for each clause ${C \in \Shorn \cup \Snhorn}$ and the corresponding
set $F_C$, the following properties are satisfied:
\begin{enumerate}\renewcommand{\theenumi}{\alph{enumi}}\renewcommand{\labelenumi}{(\theenumi)}
    \item\label{post:D-implied} ${\N \models \neg F_C \vee C}$, and

    \item\label{post:A-subset} ${F_C \subseteq \A}$.
\end{enumerate}

For the base case, consider an arbitrary clause ${C \in \S}$. If ${C \in \N}$,
we define ${F_C = \emptyset}$; otherwise, we have ${C \in \A \setminus \N}$,
so $C$ is a fact, and we define ${F_C = \{ C \}}$. Properties
\eqref{post:D-implied} and \eqref{post:A-subset} are clearly satisfied.

For the induction step, assume that the two properties are satisfied for each
clause ${C \in \Shorn \cup \Snhorn}$ at some point in time. We consider the
following two ways in which Procedure \ref{alg:Hornification} can extend
$\Shorn$ or $\Snhorn$.
\begin{itemize}
    \item Assume that resolution is applied to clauses ${C_1 = D_1 \vee A_1}$
    and ${C_2 = D_2 \vee \neg A_2}$, deriving clause ${C = D_1\sigma \vee
    D_2\sigma}$. Let ${F_C = F_{C_1} \cup F_{C_2}}$, so property
    \eqref{post:A-subset} is clearly satisfied. By induction assumption, we
    have ${\N \models \neg F_{C_1} \vee D_1 \vee A_1}$ and ${\N \models \neg
    F_{C_2} \vee D_2 \vee \neg A_2}$. By the soundness of binary resolution,
    we have ${\{ D_1 \vee A_1, \;\; D_2 \vee \neg A_2 \} \models D_1\sigma
    \vee D_2\sigma}$. But then, since ${\neg F_{C_1}}$ and ${\neg F_{C_2}}$
    contain only constants, we have ${\N \models \neg F_{C_1} \vee \neg
    F_{C_2} \vee D_1\sigma \vee D_2\sigma}$, as required for
    \eqref{post:D-implied}.
    
    \item Assume that positive factoring is applied to a clause ${C_1 = D_1
    \vee A_1 \vee B_1}$, deriving clause ${C = D_1\sigma \vee A_1\sigma}$. Let
    ${F_C = F_{C_1}}$, so property \eqref{post:A-subset} is clearly satisfied.
    By induction assumption, we have ${\N \models \neg F_{C_1} \vee D_1 \vee
    A_1 \vee B_1}$. By the soundness of positive factoring, we have ${\{ D_1
    \vee A_1 \vee B_1 \} \models D_1\sigma \vee A_1\sigma}$. But then, since
    ${\neg F_{C_1}}$ contains only constants, we have ${\N \models \neg
    F_{C_1} \vee D_1\sigma \vee A_1\sigma}$, as required for
    \eqref{post:D-implied}.
\end{itemize}

We now show the main claim of this theorem. To this end, consider an arbitrary
Horn clause $C$ such that ${\N \cup \A \models C}$. By Theorem
\ref{th:Del-Val}, at some point in time during the application of Procedure
\ref{alg:Hornification} to $\S$, we have ${\Shorn \models C}$. Note that
$\Shorn$ is a finite set.

Consider an arbitrary Horn clause ${D \in \Shorn}$. By property
\eqref{post:D-implied}, we have ${\N \models \neg F_D \vee D}$. Furthermore,
${\neg F_D \vee D}$ is a Horn clause, so by Theorem \ref{th:Del-Val}, at some
point in time time during the application of Procedure \ref{alg:Hornification}
to $\N$, we have ${\Nhorn^D \models \neg F_D \vee D}$. Finally, by property
\eqref{post:A-subset}, we have ${F_D \subseteq \A}$. These observations now
imply that ${\Nhorn^D \cup \A \models D}$.

Now let ${\Nhorn' = \bigcup_{D \in \Shorn} \Nhorn^D}$; clearly, ${\Nhorn' \cup
\A \models \Shorn}$. Note that Procedure \ref{alg:Hornification} is monotonic
in the sense that, if ${\Nhorn \models E}$ at some point in time for some
clause $E$, then this also holds at all future points in time. Furthermore,
$\Nhorn'$ is finite, so at some point in time during the application of
Procedure \ref{alg:Hornification} to $\N$, we have ${\Nhorn \models \Nhorn'}$.
By the observations from the previous paragraph, we then have ${\Nhorn \cup \A
\models \Shorn}$ as well, which implies ${\Nhorn \cup \A \models C}$, as
required.
\end{proof}

\compileNMP*

\begin{proof}
We prove by induction on the application of the inference rules in Procedure
\ref{alg:Hornification} that, at any point in time, ${\Phorn \cup \Pnhorn}$ is
a nearly-monadic program. The base case is clearly satisfied since $\P$ is
nearly-monadic. For the induction base, we consider the possible inferences
that can derive a clause in ${\Phorn \cup \Pnhorn}$. First, note that positive
factoring is never applicable to a clause of type \ref{cond:rol} from
Definition \ref{def:NMP}; furthermore, when applied to a clause of type
\ref{cond:mon}, positive factoring always produces a clause of the same type.
Second, since clauses of type \ref{cond:rol} are Horn, binary resolution can
be applied only if at least one clause is of type \ref{cond:mon}, and the
resolvent is then clearly of type \ref{cond:mon} as well.
\end{proof}

\SHIrewriting*

\begin{proof}
Consider an arbitrary ABox $\A$ and an arbitrary fact $\alpha$. By Theorem
\ref{th:transitivity}, we have that ${\T \cup \A \models \alpha}$ if and only
if ${\Omega_\T \cup \Xi_\T(\A) \models \alpha}$. By Theorem \ref{th:DLtoDD},
the latter holds if and only if ${\DD(\Omega_\T) \cup \Xi_\T(\A) \models
\alpha}$. Moreover, since $\alpha$ is a Horn clause, by Theorem
\ref{th:postpone-ABox}, the latter holds if and only if ${\Phorn \cup
\Xi_\T(\A) \models \alpha}$. We now show that the latter holds if and only if
${\Phorn \cup \Xi_\T \cup \A \models \alpha}$. Clearly, ${\Phorn \cup
\Xi_\T(\A) \models \alpha}$ implies ${\Phorn \cup \Xi_\T \cup \A \models
\alpha}$ by monotonicity of first-order logic, so we next focus on showing
that ${\Phorn \cup \Xi_\T(\A) \not\models \alpha}$ implies ${\Phorn \cup
\Xi_\T \cup \A \not\models \alpha}$.

By Theorem \ref{th:DLtoDD}, Lemma \ref{lem:compile-NMP}, and the fact that
Procedure \ref{alg:Hornification} is sound, program $\Phorn$ is nearly-monadic
and ${\Omega_\T \models \Phorn}$. Now let $\Pmon[\Phorn]$ and $\Prol[\Phorn]$
be the subsets of $\Phorn$ of the rules of type \ref{cond:mon} and
\ref{cond:rol}, respectively. Since ${\Omega_\T \models \Prol[\Phorn]}$, by
the definition of $\Xi_\T$ we have ${\Xi_\T \models \Prol[\Phorn]}$.
Furthermore, if a role atom occurs in the head of a rule in $\Pmon[\Phorn]$,
the atom is of the form $R(z,z)$; hence, each fact involving a role atom in
${\Phorn(\Xi_\T(\A)) \setminus \Xi_\T(\A)}$ is necessarily of the form
$R(c,c)$. But then, such facts clearly cannot trigger a transitivity rule in
$\Xi_\T$ to derive a new fact; furthermore, for each rule ${r \in \Xi_\T}$ of
the form ${R(x,y) \rightarrow S(x,y)}$ or ${R(x,y) \rightarrow S(y,x)}$, we
have ${\Prol[\Phorn] \models r}$; consequently, ${\Xi_\T(\Phorn(\Xi_\T(\A))) =
\Phorn(\Xi_\T(\A))}$, and the property holds.

Thus, ${\T \cup \A \models \alpha}$ if and only if ${\Phorn \cup \Xi_\T \cup
\A \models \alpha}$ for arbitrary fact $\alpha$; but then, for an arbitrary
ground query $Q$, we also have ${\T \cup \A \models Q}$ if and only if
${\Phorn \cup \Xi_\T \cup \A \models Q}$, as required.
\end{proof}  

\section{Proofs for Section \ref{sec:Termination}}\label{sec:Proofs-Termination}

\procwithcondensation*

\begin{proof}
Assume that Procedure \ref{alg:Hornification} derives a clause $C$ in line 5,
and let $D$ be the condensation of $C$. Since Procedure
\ref{alg:Hornification} is sound, we have ${\Shorn \cup \Snhorn \models C}$;
furthermore, since $C$ subsumes $D$, we have ${\{ C \} \models D}$; but then,
we have ${\Shorn \cup \Snhorn \models D}$ as well. It is therefore safe to add
$D$ to $\Shorn$ or $\Snhorn$, so let us assume that Procedure
\ref{alg:Hornification} does so; but then, this makes $C$ redundant since $D$
subsumes $C$ by the definition of condensation.
\end{proof}

\terminationsimple*

\begin{proof}
Let ${\P = \Pmon \cup \Prol}$. Since $\P$ is simple, each rule in $\Pmon$ is
of the form \eqref{eq:simple-rule} with each variable $y_i$ occurring at most
once in the rule, and each rule in $\Prol$ is of the form \eqref{eq:ria-1} or
\eqref{eq:ria-2}.
\begin{align}
    \label{eq:simple-rule}  \bigwedge A_i(x) \wedge \bigwedge R_i(x,x) \wedge \bigwedge S_i(x,y_i) \wedge \bigwedge T_i(y_i,x)  & \rightarrow \bigvee U_i(x,x) \vee \bigvee B_i(x) \\
    \label{eq:ria-1}        R(x,y)                                                                                              & \rightarrow S(x,y) \\
    \label{eq:ria-2}        R(x,y)                                                                                              & \rightarrow S(y,x)
\end{align}
It is now straightforward to check that Procedure \ref{alg:Hornification}
derives only rules of such form: positive factoring is never applicable to a
rule of the form \eqref{eq:simple-rule}--\eqref{eq:ria-2}, and binary
resolution clearly derives only rules of these forms.

Now let $C$ be an arbitrary rule derived in line 5 of Procedure
\ref{alg:Hornification}, and let $D$ be the condensation of $C$; furthermore,
let $n$ be the number of binary atoms occurring in $\P$. Since each variable
in $C$ occurs at most once in the rule, there can be at most $2n$ atoms of the
form $R(x,y_i)$ or $R(y_i,x)$ different up to variable renaming; therefore,
$D$ contains at most $2n$ variables $y_i$. Since the number of predicates in
$D$ is linear in the size of $\P$, the size of each clause is linear in the
size of $\P$ as well. But then, there can be at most exponentially many
different clauses in ${\Phorn \cup \Pnhorn}$, which implies termination of
Procedure \ref{alg:Hornification} using the standard argument
\cite{hms07reasoning}.
\end{proof}

\section{Proofs for Section \ref{sec:LimitsStrong}}\label{sec:Proofs-LimitsStrong}

We first present a well-known characterisation of the entailment of a datalog
rule from a first-order theory. The proof of Proposition
\ref{prop:criterion-entailment} is straightforward and can be found, for
example, in the work by \citeA{gmsh12completeness-guarantees}.

\begin{proposition}\label{prop:criterion-entailment}
    Let $\F$ be a set of first-order sentences, and let $r$ be a datalog rule
    of the form ${C_1 \wedge \ldots \wedge C_n \rightarrow H}$. Then, for each
    substitution $\sigma$ mapping each variable in $r$ to a distinct
    individual not occurring in $\F$ or $r$, we have ${\F \models r}$ if and
    only if
    \begin{align}
        \mathcal{F} \cup \{ \sigma(C_1), \ldots, \sigma(C_n) \} \models \sigma(H).
    \end{align}
\end{proposition}

We are now ready to prove Theorem \ref{th:negative-model-preserving}.

\negativemodelpreserving*

\begin{proof}
Let ${\Q = G(x_1)}$ be a ground query, and let $\T$ be the $\ELU$-TBox
corresponding to the program $\P$ from Example \ref{ex:non-termination-1};
thus, $\T$ consists of axioms \eqref{eq:non-horn-1}--\eqref{eq:horn-3}, which
are translated into disjunctive rules as shown below.
\begin{align}
    \label{eq:non-horn-1}   \top        & \sqsubseteq G \sqcup B    & \rightsquigarrow && \top                      & \rightarrow G(x) \vee B(x) \\ 
    \label{eq:horn-2}       \exists E.G & \sqsubseteq B             & \rightsquigarrow && E(x_1,x_0) \wedge G(x_0)  & \rightarrow B(x_1) \\
    \label{eq:horn-3}       \exists E.B & \sqsubseteq G             & \rightsquigarrow && E(x_1,x_0) \wedge B(x_0)  & \rightarrow G(x_1)
\end{align}

An individual $v$ is \emph{reachable} from an individual $w$ by a path of
length $n$ in an ABox $\A$ if individuals ${u_n, u_{n-1}, \ldots, u_0}$ exist
such that ${E(u_i,u_{i-1}) \in \A}$ for each ${1 \leq i \leq n}$, ${u_n = v}$,
and ${u_0 = w}$. In this proof, we consider 0 to be an even number. We next
prove the following property $(\ast)$, which characterises the answers to $\Q$
on ${\T \cup \A}$:
\begin{quote}
    For each ABox $\A$ containing only the $E$ predicate and for each
    individual $v$, we have ${v \in \cert{\Q}{\T}{\A}}$ iff an individual $w$
    exists such that $v$ is reachable from $w$ by a path of positive even
    length and a path of positive odd length.
\end{quote}

\smallskip

\noindent
(Proof of $\ast$, direction $\Rightarrow$) Let $v$ and $w$ be arbitrary
individuals such that $v$ is reachable from $w$ by a path of even length and a
path of odd length; thus, $\A$ contains sets of assertions of the following
form, where $k$ is a positive even number, $\ell$ is a positive odd number,
${u_k = u'_\ell = v}$, and ${u_0 = u'_0 = w}$:
\begin{align}
    \label{eq:odd-path}  \{ E(u_k,u_{k-1}), \ldots, E(u_1,u_0)              & \} \subseteq \A \\
    \label{eq:even-path} \{ E(u'_\ell,u'_{\ell-1}), \ldots, E(u'_1,u'_0)    & \} \subseteq \A
\end{align}
Let $I$ be an arbitrary model of ${\T \cup \A}$. Due to axiom
\eqref{eq:non-horn-1}, we have the following two possibilities.
\begin{itemize}
    \item Assume that ${w \in G^I}$. Then, axioms \eqref{eq:horn-2} and
    \eqref{eq:horn-3} and the assertions in \eqref{eq:odd-path} ensure that
    ${u_j \in G^I}$ for each even number ${0 \leq j \leq k}$ and ${u_i \in
    B^I}$ for each odd number ${1 \leq i \leq k-1}$; thus, we have ${v \in
    G^I}$. Furthermore, axioms \eqref{eq:horn-2} and \eqref{eq:horn-3} and the
    assertions in \eqref{eq:even-path} ensure that ${u'_i \in G^I}$ for each
    even number ${0 \leq i \leq \ell-1}$ and ${u'_j \in B^I}$ for each odd
    number ${1 \leq j \leq \ell}$; thus, we have ${v \in B^I}$. Consequently,
    we have ${v \in B^I \cap G^I}$.

    \item Assume that ${w \in B^I}$. By a symmetric argument we also conclude
    that ${v \in B^I \cap G^I}$.
\end{itemize}
Thus, we have ${v \in B^I \cap G^I}$ for an arbitrary model $I$ of ${\T \cup
\A}$, so ${v \in \cert{\Q}{\T}{\A}}$, as desired.

\smallskip

\noindent
(Proof of $\ast$, direction $\Leftarrow$) Assume that ${v \in
\cert{\Q}{\T}{\A}}$; furthermore, for the sake of contradiction assume that,
for each individual $w$ occurring in $\A$, each path from $w$ to $v$ in $\A$
is of odd length, or each path from $w$ to $v$ in $\A$ is of even length. Let
$I$ be the interpretation defined as follows:
\begin{itemize}
    \item $\Delta^I$ contains all individuals in $\A$;

    \item ${B^I = \{ w \mid \text{each path from } w \text{ to } v \text{ in }
    \A \text{ is of even length} \} \cup \{ v \}}$;

    \item ${G^I = \{ w \mid \text{each path from } w \text{ to } v \text{ in }
    \A \text{ is of odd length} \}}$; and

    \item ${E^I = \{ \tuple{c,d} \mid E(c,d) \in \A \}}$.
\end{itemize}
If there is no path from an individual $w$ to individual $v$ in $\A$, then
each path from $w$ to $v$ in $\A$ is (vacuously) of both even and odd length,
so ${w \in B^I \cap G^I}$; hence, axioms
\eqref{eq:non-horn-1}--\eqref{eq:horn-3} are satisfied for such $w$.
Furthermore, if $w_1$ is an individual such that each path from $w_1$ to $v$
in $\A$ is of even length, and if $w_2$ satisfies the same property, then each
path from $w_1$ to $w_2$ is also of even length; hence, axioms
\eqref{eq:non-horn-1}--\eqref{eq:horn-3} are satisfied for such $w_1$ and
$w_2$. Finally, if $w_1$ is an individual such that each path from $w_1$ to
$v$ in $\A$ is of odd length, and if $w_2$ satisfies the same property, then
each path from $w_1$ to $w_2$ is also of even length; hence, axioms
\eqref{eq:non-horn-1}--\eqref{eq:horn-3} are satisfied for such $w_1$ and
$w_2$. Thus, have have ${I \models \T \cup \A}$; however, ${v \not\in G^I}$,
which is a contradiction.

\medskip

This completes the proof of property ($\ast$). Now let $\P$ be the following
datalog program, where $\mathit{odd}$ and $\mathit{even}$ are fresh binary
predicates:
\begin{align}
    \label{eq:rew-1} E(x_1,x_0)                                     & \rightarrow \mathit{odd}(x_1,x_0) \\
    \label{eq:rew-2} \mathit{odd}(x_2,x_1) \wedge E(x_1,x_0)        & \rightarrow \mathit{even}(x_2,x_0) \\
    \label{eq:rew-3} \mathit{even}(x_2,x_1) \wedge E(x_1,x_0)       & \rightarrow \mathit{odd}(x_2,x_0) \\
    \label{eq:rew-4} \mathit{odd}(x,y) \wedge \mathit{even}(x,y)    & \rightarrow G(x) \\
    \label{eq:rew-5} E(x_1,x_0) \wedge G(x_0)                       & \rightarrow B(x_1) \\
    \label{eq:rew-6} E(x_1,x_0) \wedge B(x_0)                       & \rightarrow G(x_1)
\end{align}
Furthermore, let $\A$ be an arbitrary ABox, and let $\A'$ be the subset of
$\A$ containing precisely the assertions involving the $E$ predicate. Due to
rules \eqref{eq:rew-1}--\eqref{eq:rew-4}, for each individual $v$ we have ${\P
\cup \A' \models G(v)}$ iff an individual $w$ exists such that $v$ is
reachable from $w$ in $\A'$ via an even and an odd path; by property ($\ast$),
the latter is the case iff ${\T \cup \A' \models G(v)}$. Rules
\eqref{eq:rew-5} and \eqref{eq:rew-6} correspond to axioms \eqref{eq:horn-2}
and \eqref{eq:horn-3}, and they merely `propagate' $G$ and $B$ from
individuals explicitly labelled with $G$ and $B$ in $\A$; hence, it should be
clear that $\P$ is a $\Q$-rewriting of $\T$. Note, however, that $\P$ is not a
strong $\Q$-rewriting of $\T$: it contains fresh predicates $\mathit{odd}$ and
$\mathit{even}$, so ${\T \not\models \P}$.

\medskip

To complete the proof, we next show that no strong $\Q$-rewriting of $\T$
exists. To this end, let $\R$ be the infinite set containing rule
\eqref{eq:entailed-rule} instantiated for each positive even number $n$.
\begin{align}
    \label{eq:entailed-rule} E(x_n,x_0) \wedge E(x_n,x_{n-1}) \wedge \ldots \wedge E(x_1,x_0) \rightarrow G(x_n)
\end{align}
It is straightforward to see that ${\T \models \R}$: one can derive all such
rules using resolution and factoring as shown in Example
\ref{ex:non-termination-1}. We next prove that $\R$ satisfies the following
two properties, which immediately imply the claim of this the theorem.
\begin{enumerate}
    \item ${\P' \models \R}$ for each strong  $\Q$-rewriting $\P'$ of $\T$.

    \item For each finite set of datalog rules $\P'$ such that ${\T \models
    \P'}$, we have ${\P' \not\models \R}$.
\end{enumerate}

\medskip

(Property 1) Assume by contradiction that a strong $\Q$-rewriting $\P'$ of
$\T$ exists such that ${\P' \not\models \R}$; then, there exist a rule ${r \in
\R}$ such that ${\P' \not\models r}$. Let ${C_1, \ldots, C_n}$ be the body
atoms of $r$, and note that the head atom of $r$ is ${Q = G(x_n)}$. Since $r$
is a datalog rule and $\P'$ is a set of first-order formulas, by Proposition
\ref{prop:criterion-entailment}, for each substitution $\sigma$ mapping each
variable in $r$ to a distinct individual, we have ${\P' \cup \{ \sigma(C_1),
\ldots \sigma(C_n) \} \not\models \sigma(Q)}$. Now let $\sigma$ be one such
arbitrarily chosen substitution, and let ${\A = \{ \sigma(C_1), \ldots
\sigma(C_n) \}}$; clearly, we have ${\P' \cup \A \not\models \sigma(Q)}$. In
contrast, ${\R \cup \A \models \sigma(Q)}$, and, due to ${\T \models \R}$, we
have ${\T \cup \A \models \sigma(Q)}$. Thus, $\P'$ is not a strong
$\Q$-rewriting of $\T$, which contradicts our assumption.

\medskip

(Property 2) Let $\P'$ be an arbitrary finite set of datalog rules such that
${\T \models \P'}$, let $m$ be the maximal number of body atoms in a rule in
$\P'$, let $n$ be the smallest even number such that ${n > m}$, and let $\A$
be the following ABox where each $v_i$ is distinct:
\begin{align}
    \A = \{ E(v_n,v_0), E(v_n,v_{n-1}), E(v_{n-1},v_{n-2}), \ldots, E(v_1,v_0) \}
\end{align}
We next show that, for each fact $\alpha$, we have ${\P' \cup \A \models
\alpha}$ iff ${\alpha \in \A}$; this clearly implies ${\P' \cup \A \not\models
G(v_n)}$, which by Proposition \ref{prop:criterion-entailment} implies ${\P'
\not\models \R}$, as required for Property 2. We proceed by contradiction, so
assume that a fact $\alpha$ exists such that ${\P' \cup \A \models \alpha}$
and ${\alpha \not\in \A}$. Then, a rule ${r \in \P'}$ of the form ${r = C_1
\wedge \ldots \wedge C_k \rightarrow H}$ and a substitution $\sigma$ exist
such that, for ${\A' = \{ \sigma(C_1), \ldots, \sigma(C_k) \}}$, we have ${\A'
\subseteq \A}$, ${\alpha = \sigma(H)}$, and ${\alpha \not\in \A}$; note that
${\R \cup \A' \models \alpha}$. We now make the following observations.
\begin{itemize}
    \item Since ${\T \models \P'}$, we have ${\T \cup \A' \models \alpha}$.

    \item Since ${k \leq m < n}$, we have ${\A' \subsetneq \A}$.
    
    \item Let ${r' \in \P'}$ be an arbitrary non-tautological rule of the form
    ${r' = C_1' \wedge \ldots \wedge C'_\ell \rightarrow H'}$. Since ${\T
    \models \P'}$, we have ${\T \models r'}$. The latter, however, is possible
    only if $H'$ is a unary atom involving the $G$ or the $B$ predicate, and
    each $C_i'$ is an atom involving the $G$, $B$, or $E$ predicate. Thus,
    either ${\alpha = B(v_i)}$ or ${\alpha = G(v_i)}$ for some integer $i$.
    
    \item By property ($\ast$), we have ${\T \cup \A' \not\models G(v_j)}$ and
    ${\T \cup \A' \not\models B(v_j)}$ for each ${n > j \geq 0}$ since each
    such individual $v_j$ is reachable from other individuals in $\A'$ by at
    most one path. Thus, we have ${i = n}$ in the previous item.
    
    \item Individual $v_n$ is reachable from $v_0$ via two paths in $\A$;
    furthermore, due to ${\A' \subsetneq \A}$, individual $v_n$ is reachable
    from $v_0$ in $\A'$ via at most one path. Therefore, by property ($\ast$),
    we have ${\T \cup \A' \not\models G(v_n)}$ and ${\T \cup \A' \not\models
    B(v_n)}$.
\end{itemize}
The above four points are clearly in contradiction, which completes our proof.
\end{proof} 

}{
}
\end{document}